\documentclass[11pt,letterpaper]{article}
\usepackage[margin=1in]{geometry}
\usepackage[T1]{fontenc}
\usepackage[utf8]{inputenc}
\usepackage{lmodern,microtype}
\usepackage{amsmath,amssymb,amsthm,booktabs,graphicx}
\usepackage[round,authoryear]{natbib}
\usepackage{xcolor}
\usepackage[colorlinks=true,linkcolor=blue!45!black,citecolor=blue!45!black,urlcolor=blue!45!black]{hyperref}
\graphicspath{{figures/}}
\newtheorem{theorem}{Theorem}

\newcommand{\E}{\mathbb{E}}
\newcommand{\Prb}{\mathbb{P}}
\newcommand{\old}{\theta_0}
\newcommand{\new}{\theta_1}
\newcommand{\eps}{\varepsilon}
\newcommand{\calD}{\mathcal{D}}

\newcommand{\NoharmRate}{99.7}
\newcommand{\IutRate}{2.6}
\newcommand{\ReuseRate}{97.5}
\newcommand{\CorrectedRate}{3.0}
\newcommand{\BlockGain}{0.0482}
\newcommand{\ScalarGain}{0.0127}
\newcommand{\BlockRatio}{3.79}
\newcommand{\CalTrials}{20000}
\newcommand{\SearchTrials}{4000}
\newcommand{\VarianceRatio}{200.2}
\newcommand{\ZeroHarmUpper}{0.092}

\title{Certifying Model Upgrades with Slice-Wise\\Non-Regression and Incumbent Fallback}
\author{Shengwei Zhang\textsuperscript{1} \quad Tao Wu\textsuperscript{2} \quad Fei Qian\textsuperscript{2}\\
\textsuperscript{1}University of Pennsylvania\\
\textsuperscript{2}Alibaba International Digital Commerce}
\date{August 28, 2026}
\hypersetup{pdfauthor={Shengwei Zhang, Tao Wu, Fei Qian},pdftitle={Certifying Model Upgrades with Slice-Wise Non-Regression and Incumbent Fallback}}

\begin{document}
\maketitle
\begin{abstract}
An updated model can improve an aggregate metric while degrading a slice that matters to a downstream user. We study checkpoint selection subject to non-regression tolerances relative to a retained incumbent. The central distinction is between failing to detect harm and certifying non-inferiority: the former can release harmful updates with high probability when evaluation is noisy. We give a reproducible release procedure that separates candidate search from independent, paired evaluation and returns the exact incumbent when certification fails. Applying established intersection--union and Learn-then-Test principles, we state finite-sample guarantees for one frozen candidate, a finite candidate library, and a prespecified testing order. A joint release decision does not require a slice-count Bonferroni penalty, although certification power can still decrease with the number of slices. In bounded-score simulations, a no-detected-harm gate releases a harmful candidate in \NoharmRate\% of trials in one 32-slice setting, compared with \IutRate\% for an exact non-inferiority gate at a 5\% target. A constructed two-block family yields larger certified utility than a scalar path under matched candidate counts. Public digits experiments, including a subsequent continuation that improves average aggregate accuracy, return the incumbent in every run because certification is underpowered. These results establish an auditable protocol and its limitations; they do not establish benefits on foundation-model or multilingual translation upgrades.
\end{abstract}

\section{Introduction}
Model quality is often summarized by an average over tasks, classes, languages, or customer populations. An average increase permits a decrease on any one of these slices. This matters when a downstream application depends on a particular capability of the deployed model, rather than on the aggregate used to choose its successor. Version compatibility is already an established problem: \citet{echterhoff2024muscle}, for example, document instance regressions when language-model bases and their downstream adapters are updated. Slice-level expectations capture a different requirement, one that can coexist with instance-level compatibility metrics.

We consider an incumbent model $\old$, an updated checkpoint $\new$, and a set of predefined slices. The objective is to choose a useful update while bounding the probability of releasing a model whose expected score falls more than a specified tolerance on any slice. A retained incumbent makes abstention operational: when the evidence is insufficient, the procedure can return the exact existing model. Weight interpolation supplies one possible family of intermediate candidates, and block-specific interpolation can enlarge that family. Neither the incumbent anchor nor interpolation is new; their usefulness depends on how candidate selection and statistical certification are combined.

Two errors can obscure this decision. First, a rule that releases a model whenever a regression is not statistically significant answers a different question from a non-inferiority test. More uncertainty makes the former rule easier to pass. Second, a confidence statement for one prespecified model can fail after searching over models on the same evaluation data. A final sweep over every slice does not repair this candidate-selection problem if the sweep itself is repeatedly reused.

We address these issues through an explicit interface between development and release. All candidate generation, layer search, tolerance setting, and prioritization use development data. Certification then evaluates a frozen candidate or library on independent paired observations. A failure ends the attempt and invokes the incumbent fallback. This interface can also certify candidates produced by continued training, replay, distillation, or multi-objective merging; these approaches are compatible with the release protocol.

Our contribution is an operational formulation and a reproducible examination of this interface. We (i) distinguish release safety from harm detection, including their different multiplicity requirements; (ii) specialize existing statistical tools to paired, incumbent-relative constraints with explicit fallback and search assumptions; and (iii) run controlled simulations and a small public-data integration study that report both accepted utility and failed certification. The statistical proofs apply classical concentration, intersection--union testing, and Learn-then-Test reasoning. We do not claim a new general testing principle, a new merging operator, or empirical superiority on large models.

\section{Problem and scope}
\label{sec:problem}
Let $\mathcal{S}=\{1,\ldots,K\}$ denote slices fixed before certification. Slice $s$ has distribution $P_s$ and a measurable, higher-is-better score $h_s(\theta,Z)$. Define
\begin{equation}
q_s(\theta)=\E_{Z\sim P_s}[h_s(\theta,Z)],\qquad
D_s(\theta)=q_s(\theta)-q_s(\old).
\end{equation}
The tolerances $\eps_s\geq0$ are selected before viewing certification results. The population feasible set is
\begin{equation}
\mathcal F=\{\theta:D_s(\theta)\geq-\eps_s\text{ for every }s\}.
\label{eq:feasible}
\end{equation}
For interpretable, predeclared weights $w_s\geq0$, $\sum_s w_s=1$, a possible utility is $U(\theta)=\sum_s w_sD_s(\theta)$. Different slice metrics are permitted in the constraints, but their aggregate needs an explicit utility conversion. Our experiments use a common score scale.

We seek a release rule $\widehat\theta=\mathcal A(\calD_{\mathrm{dev}},\calD_{\mathrm{cert}})$ such that
\begin{equation}
\Prb\left\{\exists s:D_s(\widehat\theta)<-\eps_s\right\}\leq\alpha.
\label{eq:guarantee}
\end{equation}
The probability is over the data and any declared algorithmic randomization. This is an unconditional repeated-sampling guarantee for the release procedure, not a posterior probability for the observed model, and not a bound conditional on an upgrade being released. Its scope is the specified scores, slices, and evaluation distributions. It does not guarantee individual predictions, undiscovered slices, or future distribution shifts.

\paragraph{Paired evaluation and the anchor.}
On certification examples $Z_{is}$, compare models using
\begin{equation}
X_{is}(\theta)=h_s(\theta,Z_{is})-h_s(\old,Z_{is}),\qquad
\widehat D_s(\theta)=\frac1{n_s}\sum_{i=1}^{n_s}X_{is}(\theta).
\label{eq:paired}
\end{equation}
At $\theta=\old$, $X_{is}(\old)=0$ by identity, so the incumbent is known to satisfy \eqref{eq:feasible}; it is not estimated to be safe from a noisy self-comparison. For stochastic generation, reuse the incumbent artifact and its evaluation outputs, or otherwise make the identity comparison explicit. An independently resampled generation from the same weights need not have zero realized score difference.

\paragraph{Assumptions for finite-sample certification.}
Conditioning on development data, each candidate and score function is fixed; certification examples within each slice are independent and identically distributed from $P_s$; and each paired difference lies in a known interval $[a_s,b_s]$ of width $W_s=b_s-a_s$. Dependence across slices is allowed. For $h_s\in[0,1]$, the generic difference range is $[-1,1]$, so $W_s=2$. Tighter structural bounds may be used if justified before seeing certification data. An observed minimum and maximum are not automatically valid population bounds.

\paragraph{Checkpoint compatibility.}
A block interpolation family takes the form
\begin{equation}
\theta_b(\lambda)=(1-\lambda_b)\theta_{0,b}+\lambda_b\theta_{1,b},
\qquad \lambda\in[0,1]^B.
\label{eq:merge}
\end{equation}
It requires compatible parameter names, shapes, and semantics. Tokenizers and token indices must also agree when they determine embedding meaning. Equal architecture alone does not establish alignment; permutation symmetries can obstruct averaging \citep{ainsworth2023rebasin}. Common checkpoint ancestry is a practical starting point, not a guarantee that every interpolated model performs well. No certification proof below assumes a smooth or monotone interpolation path.

\section{From candidate search to a defensible release}
\label{sec:release}
\subsection{Three decisions that must be distinguished}
For a scalar comparison with standard error $\sigma$ and a positive normal quantile $z$, the following decisions differ:
\begin{align}
\text{Point estimate: }&\widehat D\geq-\eps;\label{eq:point}\\
\text{No detected harm: }&\widehat D+z\sigma\geq-\eps;\label{eq:noharm}\\
\text{Non-inferiority certificate: }&\widehat D-z\sigma\geq-\eps.\label{eq:ni}
\end{align}
Equation~\eqref{eq:noharm} rules out updates only when an upper confidence bound lies below the tolerance. It can control false alarms about safe models. Equation~\eqref{eq:ni} requires a lower bound to clear the tolerance and can control falsely certifying unsafe models. Failure to reject inferiority does not establish non-inferiority.

For illustration, assume an exactly normal estimator with $D=-0.20$, $\eps=0.05$, $\sigma=0.10$, and $z=\Phi^{-1}(0.95)$. The no-detected-harm rule releases with probability about $0.558$, while the non-inferiority rule releases with probability about $0.00083$. This is a calculation in the specified normal model, not an assumption about arbitrary evaluation metrics. Increasing the quantile to obtain simultaneous harm detection makes \eqref{eq:noharm} easier to pass, which can increase release risk.

\subsection{One frozen candidate: an intersection--union test}
For each slice, test the null $H_s:D_s(\theta)\leq-\eps_s$ with a valid one-sided $p$-value $p_s$. Validity means $\Prb(p_s\leq t)\leq t$ under $H_s$, for every $t\in[0,1]$. The unsafe null for the whole candidate is a union, $H=\bigcup_s H_s$. Accordingly, define
\begin{equation}
p(\theta)=\max_s p_s(\theta),\qquad
\text{certify }\theta\text{ only if }p(\theta)\leq\alpha.
\label{eq:iut}
\end{equation}
The use of a closed null at the tolerance boundary is conservative for \eqref{eq:guarantee}, which counts strict violations.

\begin{theorem}[Frozen-candidate release; standard IUT application]
\label{thm:one}
Suppose development data select one candidate independently of the certification observations. If the slice $p$-values are conditionally valid given development data, then releasing the candidate only when \eqref{eq:iut} holds, and otherwise returning $\old$, satisfies \eqref{eq:guarantee}. No independence across slices and no division of $\alpha$ by $K$ are required.
\end{theorem}
\begin{proof}
Condition on development data. If the selected candidate is strictly unsafe, fix one violating slice $s^*$. Its null is true, and the event of release is contained in $\{p_{s^*}\leq\alpha\}$, whose probability is at most $\alpha$. A feasible candidate or the exact incumbent cannot cause the event in \eqref{eq:guarantee}. Integrate over development data.
\end{proof}

This familiar intersection--union distinction also appears in multi-endpoint testing: establishing a conjunction of endpoint requirements differs from selecting one successful endpoint \citep{fda2022}. Theorem~\ref{thm:one} concerns a joint decision. It does not assert simultaneous numerical coverage of every marginal confidence bound. If an application needs all $K$ intervals to cover jointly, an appropriate simultaneous procedure is still needed.

\paragraph{A bounded-score implementation.}
Hoeffding's inequality \citep{hoeffding1963} gives
\begin{equation}
L_s(\eta)=\widehat D_s-W_s\sqrt{\frac{\log(1/\eta)}{2n_s}}.
\label{eq:hoeffding}
\end{equation}
Testing all $L_s(\alpha)\geq-\eps_s$ implements Theorem~\ref{thm:one}. For general scores in $[0,1]$, the radius is $\sqrt{2\log(1/\alpha)/n_s}$. Exact tests for special score distributions, or other valid bounds, can replace Hoeffding. A plug-in normal or bootstrap interval is not automatically a finite-sample certificate.

\subsection{Several candidates: multiplicity belongs to the search}
Let development data freeze $J$ candidates $\theta^{(1)},\ldots,\theta^{(J)}$. The same certification observations may evaluate all of them; resulting $p$-values can be dependent. Testing each at level $\alpha$ and selecting any passing candidate does not generally preserve \eqref{eq:guarantee}.

\begin{theorem}[Finite library and fixed order; LTT applications]
\label{thm:many}
Under the conditional validity assumption of Theorem~\ref{thm:one}, either procedure below satisfies \eqref{eq:guarantee}.
\begin{enumerate}
\item Assign nonnegative levels $\eta_j$ using development data, with $\sum_j\eta_j\leq\alpha$. Certify candidate $j$ when $p(\theta^{(j)})\leq\eta_j$; select any certified candidate or return $\old$.
\item Freeze an order of candidates using development data. Test in that order at level $\alpha$, stop at the first failure, and select among candidates certified before that failure, or return $\old$.
\end{enumerate}
\end{theorem}
\begin{proof}
For the first procedure, each unsafe candidate is falsely certified with probability at most $\eta_j$. A union bound over unsafe candidates gives at most $\sum_j\eta_j\leq\alpha$. For the second, condition on development data and identify the first candidate in the frozen order whose whole-candidate null is true. Certifying any candidate with a true null requires rejecting that first true null, an event of probability at most $\alpha$. If no candidate has a true null, every candidate is strictly feasible. The incumbent fallback is exact.
\end{proof}

The equal allocation is $\eta_j=\alpha/J$, with no additional $K$ factor for the joint decision. Fixed-sequence validity does not require population feasibility to be monotone in the order; a poor order can instead reduce power by causing early termination. These are instances of the Learn-then-Test framework \citep{angelopoulos2022ltt}, not new multiple-testing results.

\subsection{Development search and failure handling}
Our simplest deployment interface freezes one candidate. Development search may evaluate a scalar grid, a block grid, coordinate proposals, or a preexisting merging method. It estimates utility and anticipates the eventual certification margin, then hands a single immutable candidate to the release test. Table~\ref{tab:algorithm} states the protocol. More expensive libraries use Theorem~\ref{thm:many}.

\begin{table}[t]
\centering\small
\begin{tabular}{p{0.10\linewidth}p{0.82\linewidth}}
\toprule
Stage & Required operation \\
\midrule
1 & Retain the exact incumbent. Before certification, define slices, scores, tolerances, utility, sampling units, and error budget. \\
2 & On development data, generate and score compatible candidates. Freeze one candidate, or freeze a library and its valid testing plan. \\
3 & Evaluate incumbent and frozen candidate(s) on the same certification examples; preserve paired observations and any cross-slice clustering. \\
4 & Apply the selected whole-candidate non-inferiority test. Release only an eligible candidate. Otherwise return the incumbent and end this attempt. \\
5 & Record the selected artifact, gate, sample sizes, and failure outcome. Use separate reporting data for descriptive performance analysis. \\
\bottomrule
\end{tabular}
\caption{Candidate generation and certification are separate operations. Repair after a failed gate starts a new statistical attempt and requires new data or a predeclared valid sequential design.}
\label{tab:algorithm}
\end{table}

Active-slice screening is allowed during development. A slice initially far from its threshold can become binding after a block update, so the release gate still checks every predefined slice. If that gate fails, detecting the failure is not the same as finding another acceptable model. We explicitly return the incumbent rather than repeatedly adjusting the candidate on certification data.

For a library of $J$ checkpoints, naive development evaluation costs $J$ full slice sweeps, with additional model-loading and generation costs. A screened block search may reduce some development measurements, but its initial sweep, number of blocks, proposal rounds, and final certification still count. We do not assert that evaluation is always cheaper than continued training. Candidate generators should be compared at matched evaluation and training budgets, then passed through the same gate.

\section{What the guarantee costs and what the search can gain}
\label{sec:geometry}
\subsection{Slice count affects power without a slice-count testing penalty}
Let a frozen candidate have positive margins $\gamma_s=D_s+\eps_s>0$. For the Hoeffding implementation, denote the radius in \eqref{eq:hoeffding} at level $\alpha$ by $r_s$. If $\gamma_s>r_s$, then
\begin{equation}
\Prb\{L_s(\alpha)<-\eps_s\}\leq
\exp\left[-\frac{2n_s(\gamma_s-r_s)^2}{W_s^2}\right].
\label{eq:power}
\end{equation}
Thus, a sufficient condition for certifying the candidate with probability at least $1-\beta$ is
\begin{equation}
n_s\geq\frac{W_s^2}{2\gamma_s^2}
\left(\sqrt{\log(1/\alpha)}+\sqrt{\log(K/\beta)}\right)^2
\quad\text{for every }s.
\label{eq:samplesize}
\end{equation}
The proof is in Appendix~\ref{app:proofs}. Equation~\eqref{eq:samplesize} separates safety level $\alpha$ from joint power $1-\beta$. It does not promise acceptance uniformly in $K$. A slice close to its tolerance can require substantial data, and exact boundary certification may have low power even with a correct implementation.

Paired measurements offer a principled way to improve precision:
\begin{equation}
\operatorname{Var}(h_s(\theta,Z)-h_s(\old,Z))=
\operatorname{Var}(h_s(\theta,Z))+\operatorname{Var}(h_s(\old,Z))
-2\operatorname{Cov}(h_s(\theta,Z),h_s(\old,Z)).
\end{equation}
Positive covariance reduces the variance of the difference relative to independent evaluations. It does not justify changing the direction of a confidence bound. Range-based Hoeffding bounds do not exploit every variance reduction; variance-sensitive valid tests are a useful extension.

\subsection{Scalar paths and block freedom}
The identity $\theta(0)=\old$ ensures that the population feasible set contains the incumbent. A scalar interpolation only explores the diagonal $\lambda_1=\cdots=\lambda_B$, whereas a block search includes off-diagonal candidates. An enlarged family cannot reduce the best population utility if the old family is included, but a finite development search and certification penalty can eliminate this advantage in practice.

A scalar ``first binding threshold'' is globally optimal only under stronger geometry assumptions. If each slice's entire feasible coefficient set equals $[0,t_s]$, and utility is nondecreasing over their intersection, then the scalar optimum is $\min_s t_s$. If a slice has disconnected feasible regions, this only describes the component attached to zero. Evaluating a finite grid and selecting its best certified member requires no connectedness assumption, although it does not certify coefficients between evaluated grid points.

Our controlled geometry has eight slices and two blocks:
\begin{equation}
D_1(\lambda)=-0.08\lambda_1+0.02\lambda_2,\qquad
D_s(\lambda)=0.10\lambda_1+0.04\lambda_2\quad(s\geq2),
\label{eq:geometry}
\end{equation}
with $\eps_s=0.01$ and uniform weights. The scalar feasible endpoint is $\lambda_1=\lambda_2=1/6$, giving utility $0.01917$. The continuous block optimum is $(0.375,1)$, giving $0.06656$. This explicitly constructed separation motivates a test of whether the search advantage survives finite-data certification. It is not evidence that fragile capabilities in natural networks are localized to particular layers.

\section{Experiments}
\label{sec:experiments}
\subsection{Design, provenance, and reporting}
All numerical results in this section were executed locally from the accompanying code. The synthetic studies use explicit bounded paired-score distributions, seed 20260911, independent random streams, \CalTrials\ independent trials for gate studies, and \SearchTrials\ trials per candidate-search configuration. Counts sampled directly from binomial distributions are exact sufficient statistics for the described observations, not a Gaussian approximation. Saved artifacts include these counts, trial-level decisions, selected coefficients, CSV summaries, generated figures, and a source-hash manifest. Rates have exact 95\% Clopper--Pearson intervals; utility intervals quantify Monte Carlo uncertainty. Main figures show representative settings; full sample-size and correlation results are retained in machine-readable tables.

The primary outcomes are the probability of a harmful release, probability of a non-incumbent release, and population utility of the deployed model. Synthetic population expectations are known by construction. The retained-gain ratio is reported only in the geometry experiment, whose candidate gain is exactly positive ($0.115$). Real-data reporting uses absolute accuracy differences and does not divide by a negative candidate gain.

\subsection{Harm detection and certification have different operating behavior}
Let $X_s\in\{-0.2,+0.2\}$ with mean $D_s$, $\eps_s=0.02$, and $n\in\{128,512,2048\}$ observations per slice. In the harmful scenario, one slice has $D_1=-0.022$ and the others have $D_s=0.08$. In the safe scenario, every slice has $D_s=0.01$. We vary $K\in\{1,8,32\}$.

The exact endpoint non-inferiority test uses $C_s\sim\operatorname{Binomial}(n,p_s)$, $p_{0}=(0.2-0.02)/0.4=0.45$, and the upper-tail $p$-value $\Prb_{p_0}(C\geq C_s)$. Exact IUT requires all these $p$-values to be at most $\alpha$. The simultaneous exact comparator uses $\alpha/K$; it represents a more demanding marginal familywise construction. The no-detected-harm gate instead rejects a model when at least one lower-tail harm test is significant at $\alpha/K$. We also evaluate the point-estimate and bounded Hoeffding IUT gates.

\begin{figure}[t]
\centering
\includegraphics[width=\linewidth]{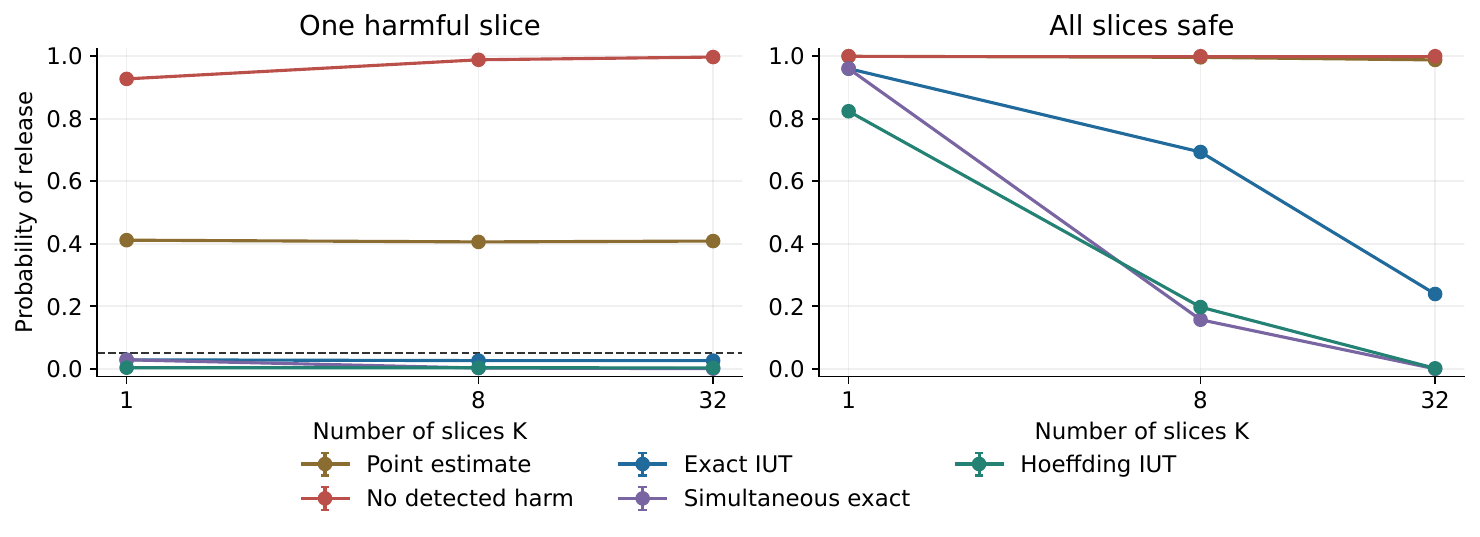}
\caption{Release probability at $n=512$ examples per slice. Left: release is an error because one slice is beyond tolerance. Right: every slice is safe, so release measures power. Error bars are exact 95\% Monte Carlo intervals, often smaller than the markers. The dashed line marks the 5\% false-release target. A simultaneous harm-detection screen increasingly accepts the harmful candidate as $K$ grows; this behavior is not a non-inferiority guarantee.}
\label{fig:calibration}
\end{figure}

Table~\ref{tab:calibration} and Figure~\ref{fig:calibration} show the difference. At $K=32,n=512$, the no-detected-harm rule releases the harmful candidate in \NoharmRate\% of trials, while exact IUT releases it in \IutRate\%. The latter is below the 5\% target, with discreteness and the nonzero violation making it conservative. The point-estimate rule has no comparable protection. The simultaneous exact and Hoeffding gates are more conservative in this family. In the safe scenario, requiring all slices to pass reduces power as $K$ increases; Theorem~\ref{thm:one} removes a testing penalty, not this conjunction effect.

\begin{table}[t]
\centering\small
\begin{tabular}{lrrr}
\toprule
Rule & $K=1$ & $K=8$ & $K=32$ \\
\midrule
Point estimate & 41.16 & 40.60 & 40.87 \\
No detected harm & 92.74 & 98.86 & 99.73 \\
Exact IUT & 2.89 & 2.63 & 2.61 \\
Simultaneous exact & 2.89 & 0.26 & 0.05 \\
Hoeffding IUT & 0.36 & 0.33 & 0.32 \\
\bottomrule
\end{tabular}

\caption{Harmful-release rates (\%) over \CalTrials\ trials at $n=512$. The true violation is 0.002 beyond the tolerance. All methods see identical draws in each configuration.}
\label{tab:calibration}
\end{table}

We also repeat the 32-slice study with 0\%, 50\%, or 100\% of example positions sharing a common latent uniform variable across slices. Thresholding this uniform at each slice's success probability preserves each binomial marginal while introducing dependence. The exact IUT harmful-release frequency remains below 5\% for each tested dependence fraction. The theorem, rather than this finite sweep, establishes validity under arbitrary cross-slice dependence satisfying the within-slice assumptions.

\subsection{Searching certification results inflates false release}
We freeze $J\in\{1,5,20,100\}$ synthetic candidates, each with a score mean $D=-0.0205$, tolerance 0.02, and $n=2048$ observations. Candidate-specific score observations are independent in this stress test. Every non-incumbent release is therefore harmful. We compare selecting any candidate that passes an uncorrected level-0.05 test, candidate Bonferroni correction, a frozen fixed sequence, and choosing one candidate on an independent development split before certification.

\begin{figure}[t]
\centering
\includegraphics[width=0.83\linewidth]{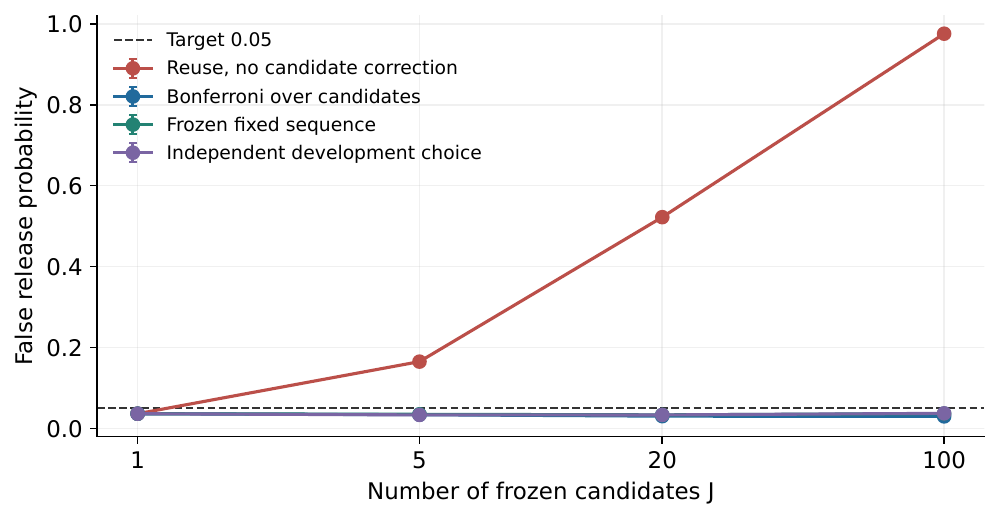}
\caption{False-release probability after searching a frozen library of harmful candidates. At $J=100$, uncorrected reuse releases harm in \ReuseRate\% of trials; candidate correction gives \CorrectedRate\%. Fixed-sequence testing and independent development choice preserve their designated error target. Candidate independence is a simulation setting, not a requirement of the validity proofs.}
\label{fig:adaptive}
\end{figure}

Figure~\ref{fig:adaptive} shows that an individually valid test does not remain valid when any success in a large search can trigger deployment. At $J=100$, the uncorrected rule reaches \ReuseRate\% harmful releases, compared with \CorrectedRate\% using $\alpha/J$. The all-harmful library makes the fixed-sequence interpretation particularly transparent: any false certification requires the first candidate to pass. Independent development selection permits searching the development data freely, while the subsequent single test retains its validity. This study isolates statistical selection effects; it does not model the natural dependence between nearby neural checkpoints.

\subsection{Block search can retain more utility in a constructed family}
For the means in \eqref{eq:geometry}, observations are
\begin{equation}
X_{is}(\lambda)=\mu_s^\top\lambda+
0.06(\lambda_1 V_{is1}+\lambda_2 V_{is2}),
\qquad V_{isb}\overset{\mathrm{iid}}{\sim}\operatorname{Rademacher}.
\label{eq:dgp}
\end{equation}
These are valid paired differences between scores in $[0,1]$: an incumbent score in $[0.48,0.52]$ can share its base noise with the updated score. Each observed coefficient lies within $[-0.2,0.2]$, which provides the declared difference width $W(\lambda)=0.4\|\lambda\|_1$. The identity candidate has width zero.

Each trial generates independent development and certification splits of $n$ examples per slice. Search maximizes development utility among candidates satisfying a development margin that anticipates the Hoeffding certification radius and a fixed design buffer (Appendix~\ref{app:simulation}). It then freezes one candidate and certifies it once. We compare a coarse 21-point scalar grid, a 441-point scalar grid, and a $21\times21$ block grid with 441 candidates. The dense scalar and block grids match candidate counts; the coarse grid is a resolution diagnostic. The block family has a known geometric advantage by construction, and this comparison is not a general claim about equal wall-clock search costs.

\begin{figure}[t]
\centering
\includegraphics[width=\linewidth]{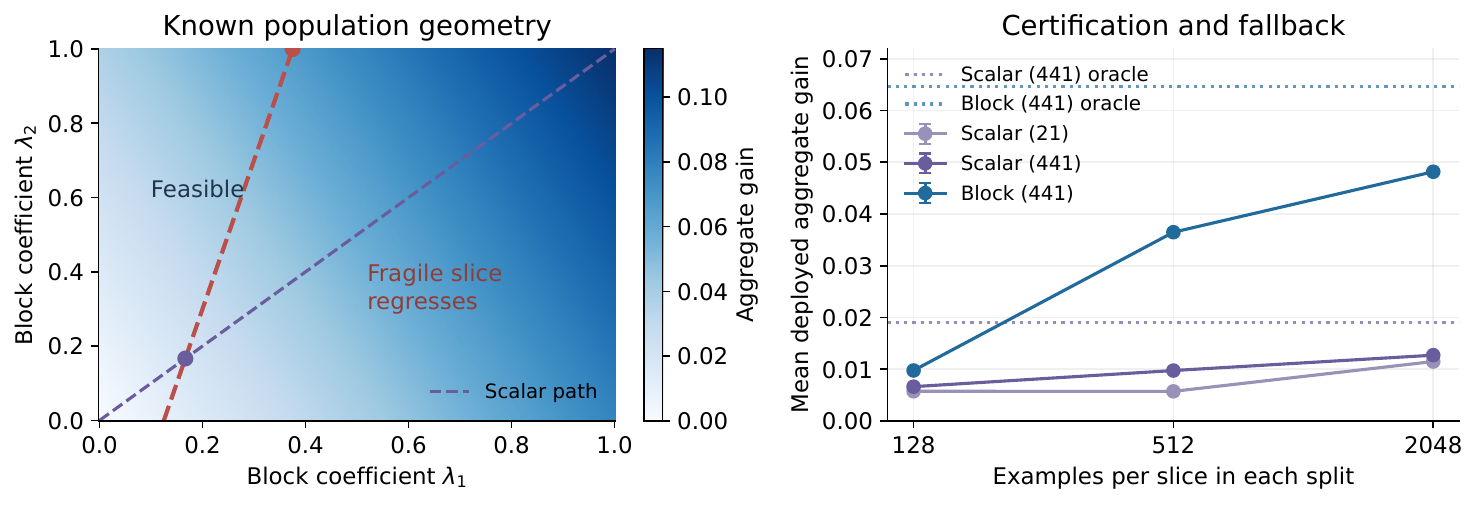}
\caption{Left: the population feasible region for the constructed two-block family, with continuous scalar and block optima marked. Right: mean deployed utility after independent certification and incumbent fallback; dotted lines give the corresponding finite-grid population optima. All search families use the same development and certification draws. Error bars show 95\% Monte Carlo mean intervals.}
\label{fig:geometry}
\end{figure}

\begin{table}[t]
\centering\small
\begin{tabular}{llrrrr}
\toprule
$n$ & Family & Gain & Retained (\%) & Fallback (\%) & Harmful releases \\
\midrule
128 & Scalar (21) & 0.0058 & 5.0 & 0.00 & 0/4000 \\
128 & Scalar (441) & 0.0066 & 5.8 & 4.83 & 0/4000 \\
128 & Block (441) & 0.0098 & 8.5 & 4.83 & 0/4000 \\
512 & Scalar (21) & 0.0057 & 5.0 & 0.22 & 0/4000 \\
512 & Scalar (441) & 0.0098 & 8.5 & 4.70 & 0/4000 \\
512 & Block (441) & 0.0365 & 31.7 & 3.52 & 0/4000 \\
2048 & Scalar (21) & 0.0115 & 10.0 & 0.00 & 0/4000 \\
2048 & Scalar (441) & 0.0127 & 11.1 & 4.45 & 0/4000 \\
2048 & Block (441) & 0.0482 & 41.9 & 2.50 & 0/4000 \\
\bottomrule
\end{tabular}

\caption{Utility after certification. Parentheses indicate candidate counts. Retained gain is relative to the known full-update gain of 0.115. Fallback is an actual output of the procedure, not an excluded trial.}
\label{tab:utility}
\end{table}

At $n=2048$, matched-count block search achieves mean utility \BlockGain, compared with \ScalarGain\ for dense scalar search, a factor of \BlockRatio\ in this constructed family. Certification failures reduce the gains actually released, and coarse scalar resolution creates additional effects at intermediate sample sizes. No harmful release occurs in any of the nine search configurations. For each configuration's \SearchTrials\ trials, zero events correspond to a two-sided 95\% upper confidence endpoint of \ZeroHarmUpper\%, not proof of zero risk. The theorem supplies the risk bound; the experiment validates the implemented behavior under this generator.

\subsection{A public digits integration test exposes the data bottleneck}
We additionally train and merge actual small neural checkpoints on the public optical digits data distributed by scikit-learn \citep{sklearn_digits,uci_digits}. Across five fixed seeds, a $64$--$48$--$10$ ReLU network is trained for 12 epochs; the update continues for 24 epochs using only classes 0--4. Both scalar and layer searches evaluate 121 development candidates. Data are split into training, development, certification, and reporting partitions before fitting. The continuation is deliberately a forgetting stress test, and its aggregate accuracy also deteriorates; it is not an aggregate-improving upgrade example.

Mean reporting accuracies are 0.9478 for the incumbent, 0.7300 for the updated checkpoint, 0.9494 for the scalar candidate, and 0.9483 for the layer candidate. The protocol freezes its proposed candidate using development data. Every proposal fails the bounded Hoeffding release gate: certification has only 35--37 observations per class and tolerance 0.03. The actual released model is the incumbent in all five runs. This result validates end-to-end fallback and illustrates severe power limits; it does not establish a useful certified upgrade, layer superiority, or a population guarantee beyond the sampling assumptions. Full configurations, descriptive seed variation, weights, split indices, and predictions are included in Appendix~\ref{app:digits} and the artifact.

After reviewing this stress test, we fixed and executed a second protocol that continues on all ten training classes. Across the same five seeds, the continued checkpoint reaches mean reporting accuracy 0.9594, compared with 0.9567 for scalar and 0.9544 for layer proposals; all frozen proposals still fail the gate. Thus a path with higher average aggregate accuracy also encounters the certification bottleneck, while merging has no mean accuracy advantage over its fully continued endpoint here. This subsequent extension reuses the original splits and is descriptive, not fresh confirmatory certification. Its unchanged 24 epochs involve more training updates than the selective path, so the comparison does not isolate class coverage at matched compute.

\section{Relation to existing work}
\paragraph{Merging and multi-objective selection.}
WiSE-FT interpolates pretrained and fine-tuned weights to improve robustness \citep{wortsman2022wise}; model soups study averaging related fine-tuned models \citep{wortsman2022soups}. We use this family of operations rather than claim a new one. MAP explicitly estimates Pareto fronts of merge coefficients, and Pareto Merging learns preference-conditioned tradeoffs \citep{li2025map,chen2025pareto}. It would therefore be incorrect to characterize all existing merging methods as optimizing only one fixed average. A relative slice constraint defines an admissible portion of a tradeoff set; our focus is how to certify a selected candidate with finite paired data. A strong empirical comparison would pass candidates from these methods through the same release interface.

\paragraph{Statistical and safe-update frameworks.}
Learn then Test provides the direct precedent for calibrating a parameterized predictor through multiple testing \citep{angelopoulos2022ltt}. Seldonian algorithms separate candidate selection from safety testing and can return no solution \citep{hoag2023seldonian}. We instantiate these ideas with an incumbent-relative difference and an exact operational fallback. SafeOpt studies optimization from safe regions under Gaussian-process regularity, including safety of queried points \citep{sui2015safeopt}; our offline evaluation points need not themselves be safe to deploy. \citet{elmecker2026safe} compute locally invariant parameter domains for safe updates, explicitly include the no-op anchor, and provide finite-sample concentration arguments. Our certificates apply to frozen evaluated candidates rather than certify a connected parameter domain. These are differences in the certified object and workflow, not a claim that prior safe-update work lacks statistical guarantees.

\paragraph{Update compatibility and evaluation.}
MUSCLE addresses LLM update compatibility through instance regressions and adapters \citep{echterhoff2024muscle}. A slice can maintain its mean while some formerly correct examples become wrong, so slice non-regression does not subsume negative-flip protection. Paired significance testing in NLP has a long history \citep{koehn2004,dror2018}; the relevant additional issue here is the use of evaluation results to make a release decision after candidate search.

\section{Limitations and conclusions}
The numerical evidence is mainly synthetic. It demonstrates that a block family can contain useful candidates and that different release gates have different error behavior under known distributions. It does not show that this geometric advantage is common, that the proposed search is computationally superior, or that language-model upgrades can be certified with realistic evaluation budgets. The public digits studies are small and yield no certified non-incumbent upgrade; the subsequent balanced path also uses previously examined holdouts and supplies descriptive evidence only.

The bounded-score theorem assumes representative independent examples within each slice. Shared documents, writers, or annotation batches can violate this assumption and require an appropriate cluster-level design. Unbounded metrics and corpus-level scores need their own statistical treatment. Repeated organizational releases consume additional error and tolerance budgets; the one-attempt guarantee here does not certify an unrestricted sequence of updates. Positive tolerances allow bounded degradation, and utility improvement requires an additional condition if it must itself be certified.

The actionable conclusion is a specific release discipline: generate candidates on development data, preserve paired incumbent comparisons, use non-inferiority in the correct direction, account for selection on certification results, and return the exact incumbent when evidence is inadequate. Low-dimensional merging can make that discipline convenient, but cannot substitute for adequate data or correct inference. The accompanying proofs and experiments make both successful selection and fallback reviewable.

\paragraph{Reproducibility statement.}
The artifact contains executable experiment scripts, independent gate tests, raw sufficient statistics, neural checkpoints, split indices, result tables, figure sources, and environment/source manifests. All numerical tables in this manuscript are generated from completed runs. No private corpus, internal deployment statistic, or unrun large-model result is included.

\paragraph{AI assistance statement.}
AI tools assisted with critique of the initial proposal, repair and checking of mathematical derivations, literature searches, experimental design and implementation, execution and inspection of local runs, and manuscript drafting. The recorded results come from executed programs. The authors are responsible for all claims, derivations, sources, code, and results, including the validation of AI-assisted content.

\begingroup
\small
\setlength{\bibsep}{4pt}
\raggedright
\bibliographystyle{plainnat}
\bibliography{references}
\endgroup

\clearpage
\appendix
\section{Proofs, boundary cases, and interpretation}
\label{app:proofs}
\subsection{Validity of the bounded paired test}
For a frozen candidate, assume $X_{is}\in[a_s,b_s]$ are iid with mean $D_s$. Hoeffding gives, for every $t>0$,
\begin{equation}
\Prb\{\widehat D_s-D_s\geq t\}\leq
\exp(-2n_st^2/W_s^2).
\end{equation}
Under $D_s\leq-\eps_s$, the event
$\widehat D_s-W_s\sqrt{\log(1/\eta)/(2n_s)}\geq-\eps_s$
implies an upward deviation of at least the radius. Its probability is at most $\eta$. Equivalently, a valid conservative $p$-value for $W_s>0$ is
\begin{equation}
p_s=\exp\left[-\frac{2n_s\{\max(\widehat D_s+\eps_s,0)\}^2}{W_s^2}\right].
\end{equation}
At the incumbent, $W_s=0$ and the difference is deterministically zero; this identity is handled separately instead of dividing by zero or constructing a random boundary test. A zero empirical variance for a distinct model does not give this identity guarantee.

Conditional validity is essential. Development data may determine model coefficients, slice weights, tolerances, or test order before certification, provided the resulting testing problem satisfies the stated assumptions conditional on that information. If certification outcomes alter those quantities, the above fixed-candidate calculation is no longer sufficient. An independent new sample or a valid sequential construction is then needed.

\subsection{Power calculation}
The failure event for a truly interior candidate is
\[
L_s(\alpha)<-\eps_s
\iff\widehat D_s-D_s<-(\gamma_s-r_s).
\]
When $\gamma_s>r_s$, the lower-tail Hoeffding bound gives \eqref{eq:power}. To make this probability at most $\beta/K$, it suffices that
\[
\gamma_s\geq\frac{W_s}{\sqrt{2n_s}}
\left(\sqrt{\log(1/\alpha)}+\sqrt{\log(K/\beta)}\right).
\]
Squaring and rearranging gives \eqref{eq:samplesize}; union-bounding the $K$ failure events yields joint acceptance probability at least $1-\beta$. For equal candidate allocation replace $\alpha$ inside the radius by $\alpha/J$. The result is sufficient, potentially conservative, and does not claim necessity or optimal sample complexity.

\subsection{When point-estimate acceptance can decay exponentially}
An illustrative Gaussian calculation can explain a many-slice power problem, but requires its assumptions to be stated correctly. Fix one candidate and assume independent errors
$\widehat D_s=D_s+\sigma_s Z_s$, $Z_s\sim N(0,1)$. Suppose at least $m$ slices have bounded standardized slack,
\begin{equation}
0\leq(D_s+\eps_s)/\sigma_s\leq B<\infty.
\end{equation}
Then
\begin{equation}
\Prb\{\text{all point-estimate constraints pass}\}
=\prod_s\Phi((D_s+\eps_s)/\sigma_s)\leq\Phi(B)^m.
\end{equation}
If $m\geq\rho K$, the bound decays exponentially in $K$. A lower bound on slack and an upper bound on variance would give the opposite inequality and cannot prove this conclusion. Correlated slices, increasing margins, or increasing sample sizes change the result. The calculation is for a fixed candidate and does not by itself establish that the coefficient selected by an arbitrary search converges to zero. In particular, an incumbent self-comparison has exactly zero paired error.

This phenomenon is a power issue, not a reason to switch to an upper-bound release screen. A valid non-inferiority procedure can also reject many truly feasible candidates when data are inadequate. The practical responses are more informative paired data, a defensible tolerance, a better candidate, or fallback.

\subsection{Feasible components and repair operators}
For scalar coefficients let $\Lambda_s=\{\lambda:D_s(\lambda)\geq-\eps_s\}$. If $\Lambda_s=[0,t_s]$ for each slice, their intersection equals $[0,\min_s t_s]$. Monotone utility on that set gives the endpoint solution. Merely containing an anchored interval is weaker. For example, with zero tolerance,
\begin{equation}
D_1(\lambda)=\lambda(\lambda-0.25)(\lambda-0.75),\qquad0\leq\lambda\leq1,
\end{equation}
the feasible set is $[0,0.25]\cup[0.75,1]$. Add a second slice $D_2(\lambda)=\lambda$ and choose a utility increasing sufficiently with $D_2$. Stopping at the first violation misses the later feasible region. A grid can reveal sampled disconnected components; it cannot rule out unseen narrow components without further regularity assumptions.

Similarly, uncontrolled repair can cycle even with convex constraints. With $C_1=(-\infty,1]$, $C_2=[-1,\infty)$, a rule that sends $2$ to $-2$ and $-2$ to $2$ always fixes the currently violated constraint while never entering their nonempty intersection. This shows that the repair specification is too weak, not that convexity is absent.

For completeness, in finite-dimensional Euclidean space suppose all $C_s$ are closed and convex, their intersection $F$ is nonempty, and a repair step is the metric projection onto the intersection of currently violated sets. Such a step is Fej\'er monotone relative to $F$:
\[
\|x_{t+1}-z\|^2\leq\|x_t-z\|^2-\|x_{t+1}-x_t\|^2\quad(z\in F).
\]
Thus the sequence is bounded and step lengths tend to zero. Every violated set has distance from $x_t$ at most that step length, while satisfied sets have distance zero. Any cluster point therefore lies in every closed $C_s$, and Fej\'er monotonicity implies convergence to that point. The successful operator uses additional geometry and projection control. We do not rely on either this property or failure of a repair baseline for our release guarantee. A repair pipeline can retain the same incumbent and use the same certification test.

\subsection{Aggregate guarantees and repeated releases}
With positive tolerances, a certified candidate can have a negative aggregate change. If aggregate non-regression is mandatory, include a separately defined aggregate difference as another endpoint and certify its lower bound as well. Appropriate mixture sampling or a valid combination of slice estimates is needed to define that endpoint; selecting the largest noisy aggregate alone does not certify improvement.

For repeated release attempts indexed by $t$, a predeclared allocation with $\sum_t\alpha_t\leq\alpha_{\mathrm{total}}$ yields a union-bound control if each attempt's conditional validity assumptions hold. Tolerances also accumulate: a sequence of individually tolerated decreases can move far below the original model. One may compare every release to a fixed reference, maintain a cumulative tolerance budget, or specify a sequential target. These are distinct policies and are not implemented by the one-step experiments here.

\section{Synthetic generators and complete implementation details}
\label{app:simulation}
\subsection{Gate calibration}
For each configuration, the script samples $C_s\sim\operatorname{Binomial}(n,(D_s+0.2)/0.4)$ and sets $\widehat D_s=0.4C_s/n-0.2$. The same $C_s$ are supplied to every gate. The no-detected-harm test accepts only if $\Prb_{0.45}(C\leq C_s)>\alpha/K$ for every slice; this is a lower-tail test for harm. Exact IUT instead requires $\Prb_{0.45}(C\geq C_s)\leq\alpha$ for every slice. Discrete ties are handled by these explicit inequalities. The Hoeffding width is 0.4.

For the dependence study, $n_{\mathrm{shared}}=\rho n$ positions use shared uniforms across slices, with $\rho\in\{0,0.5,1\}$. In the one-harmful setting the two thresholds are $p_{\mathrm{low}}=0.445$ and $p_{\mathrm{high}}=0.70$. Counts in the shared positions come from a multinomial distribution with probabilities
$(p_{\mathrm{low}},p_{\mathrm{high}}-p_{\mathrm{low}},1-p_{\mathrm{high}})$. The harmful slice uses the first cell, and each easy slice uses the first two. Remaining positions use independent binomials. In the all-safe setting all slices share the same threshold 0.525. Within each slice, all positions remain iid with the specified marginal; cross-slice dependence changes.

\begin{table}[ht]
\centering\small
\begin{tabular}{rrrrrr}
\toprule
$n$ & $K$ & Harm: IUT & Harm: Hoeffding & Safe: IUT & Safe: simultaneous \\
\midrule
128 & 1 & 3.02 & 0.45 & 48.17 & 48.17 \\
128 & 8 & 3.07 & 0.53 & 0.39 & 0.00 \\
128 & 32 & 3.04 & 0.53 & 0.00 & 0.00 \\
512 & 1 & 2.89 & 0.36 & 96.03 & 96.03 \\
512 & 8 & 2.63 & 0.33 & 69.36 & 15.71 \\
512 & 32 & 2.61 & 0.32 & 23.94 & 0.00 \\
2048 & 1 & 1.67 & 0.20 & 100.00 & 100.00 \\
2048 & 8 & 1.65 & 0.19 & 100.00 & 99.99 \\
2048 & 32 & 1.57 & 0.18 & 100.00 & 99.74 \\
\bottomrule
\end{tabular}

\caption{Completed sample-size study: release rates (\%) over \CalTrials\ trials per configuration. ``Harm'' is the one-harmful-slice scenario, so release is an error; ``Safe'' has every slice at mean 0.01, so release measures power. IUT uses exact binomial tests. The last column uses exact marginal tests at level $\alpha/K$.}
\end{table}

\begin{table}[ht]
\centering\small
\begin{tabular}{lrrrr}
\toprule
Shared fraction & Harm: no detection & Harm: IUT & Safe: IUT & Safe: simultaneous \\
\midrule
0.0 & 99.81 & 2.90 & 24.19 & 0.00 \\
0.5 & 99.74 & 2.79 & 62.72 & 8.52 \\
1.0 & 99.78 & 2.79 & 95.66 & 64.92 \\
\bottomrule
\end{tabular}

\caption{Release rates (\%) with shared latent observations, $K=32,n=512$. Marginal score distributions are unchanged. Strong dependence can substantially change joint acceptance without invalidating the IUT argument.}
\end{table}

\subsection{Adaptive-selection stress test}
Candidate $j$ has independent two-point scores with mean $-0.0205$ and $n=2048$; all tolerances are 0.02. Certification counts have dimensions trials $\times J$. Development counts are generated independently with the same distribution, and select the candidate with largest development mean, breaking ties by the first index. The fixed order is the numerical candidate index and is frozen before any counts are observed. The fixed-sequence implementation takes cumulative products of pass indicators, so a failure can never be followed by another valid certification. All raw development and certification counts are saved.

\subsection{Utility selection}
For each of eight slices, the two Rademacher coordinate means in \eqref{eq:dgp} are represented by independent $\operatorname{Binomial}(n,1/2)$ counts. The same per-slice observations determine all candidate scores within a split, preserving their strong candidate dependence. Independent streams generate the development and certification splits.

Let $r_n(\lambda)=0.4\|\lambda\|_1\sqrt{\log(1/\alpha)/(2n)}$. Development search solves the finite problem
\begin{equation}
\max_{\lambda\in\mathcal G}\widehat U_{\mathrm{dev}}(\lambda)
\quad\text{subject to}\quad
\widehat D_{s,\mathrm{dev}}(\lambda)-r_n(\lambda)
-\frac{0.12\|\lambda\|_2}{\sqrt n}\geq-0.01\quad\forall s.
\label{eq:devsearch}
\end{equation}
The last term is a fixed anticipatory design buffer, not a claimed optimal acquisition rule. Certification uses only $r_n$ and fresh observations. The incumbent is explicitly included. All families share the same trials, scores, tolerances, buffer, and gate; only the candidate grid changes. The 21-point scalar grid has coordinates $j/20$, the 441-point scalar grid has coordinates $j/440$, and the block grid is their $21\times21$ coarse Cartesian product. Search is exhaustive on these small grids; no neural layer localization or coordinate-ascent convergence is inferred.

\subsection{Pairing diagnostic}
To isolate variance, define independent Rademacher variables $U,V$ and scores
\[
h(\old,Z)=0.5+0.20U,\qquad
h(\theta,Z)=0.5+0.20U+0.01+0.02V.
\]
Both scores lie in $[0,1]$. A paired mean difference from $n=256$ observations has variance $0.02^2/n$. Independent incumbent and candidate examples instead give variance $(2\cdot0.20^2+0.02^2)/n$, a factor of 201. The executed simulation gives an empirical factor of \VarianceRatio. This diagnostic verifies the benefit under positive shared covariance, not universal superiority of pairing for every data-generating process.

\subsection{Saved files and statistical uncertainty}
The \texttt{base\_upgrade.py} script writes calibration, correlation, adaptive-selection, utility, and variance CSV files. Compressed NPZ files store the original sufficient statistics needed to recompute the gates. The utility trial file records each proposed and released coefficient pair, certification outcome, utility, and true violation indicator. \texttt{manifest.json} records software versions, the root seed, settings, and hashes of generated results; figure PDFs and PNGs are generated by the same script.

For a count $c$ of events among $R$ independent trials, the reported two-sided 95\% interval is
\[
\left[\operatorname{Beta}^{-1}(0.025;c,R-c+1),
\operatorname{Beta}^{-1}(0.975;c+1,R-c)\right],
\]
with endpoints 0 or 1 at the corresponding extremes. Utility intervals use mean $\pm1.96$ Monte Carlo standard errors. These intervals describe repeated simulation, not uncertainty in a model's population score within one trial. Random data uncertainty and Monte Carlo uncertainty are different layers.

\section{Public digits checkpoint study}
\label{app:digits}
\subsection{Frozen training and evaluation protocol}
The data contain 1,797 examples, 64 features, and ten digit classes. Features are divided by 16. Stratified splits allocate 718 examples to training, 359 to development, 360 to certification, and 360 to reporting. The five seeds are 1701--1705. A $64$--$48$--$10$ network uses a ReLU hidden layer, mini-batch size 128, learning rate 0.15, and momentum 0.9. Incumbent training runs 12 epochs, followed by 24 continuation epochs on classes 0--4. The update shares exact parameter ancestry with the incumbent.

The fixed development objective assigns each class 0--4 weight 0.16 and each class 5--9 weight 0.04. Each slice is a class and uses accuracy. Scalar interpolation evaluates 121 coefficients in $[0,1]$; layer interpolation evaluates an $11\times11$ grid, grouping each affine layer's weight and bias together. Both use the development constraint $\widehat D_s\geq-0.03$. The proposal policy then freezes a candidate using development results before applying a single all-class Hoeffding gate with $\alpha=0.05$, $\eps_s=0.03$, and paired accuracy differences in $[-1,1]$. Reporting results do not alter training, search, or the gate.

The dataset source and content hash are recorded; features are reloaded from the documented scikit-learn dataset. All split indices, trained checkpoints, model predictions, per-class counts, proposal coefficients, and pass/fail decisions are saved. A finite-difference check of the implemented network gradient gave maximum error below $4.1\times10^{-11}$. The simulation seed uncertainty is not reused as a formal population confidence interval.

\begin{table}[ht]
\centering\small
\begin{tabular}{lrr}
\toprule
Model & Reporting accuracy & Weighted reporting accuracy \\
\midrule
Incumbent & $0.9478\pm0.0077$ & $0.9479\pm0.0098$ \\
Continued checkpoint & $0.7300\pm0.0202$ & $0.8813\pm0.0113$ \\
Scalar proposal & $0.9494\pm0.0087$ & $0.9544\pm0.0123$ \\
Layer proposal & $0.9483\pm0.0089$ & $0.9546\pm0.0115$ \\
Actually released & $0.9478\pm0.0077$ & $0.9479\pm0.0098$ \\
\bottomrule
\end{tabular}
\caption{Mean $\pm$ descriptive sample standard deviation across five fixed split/training seeds. These are not five independent population confidence intervals. All five release attempts return the incumbent; the proposal rows are descriptive and are not deployed results.}
\label{tab:digits}
\end{table}

\begin{figure}[ht]
\centering
\includegraphics[width=\linewidth]{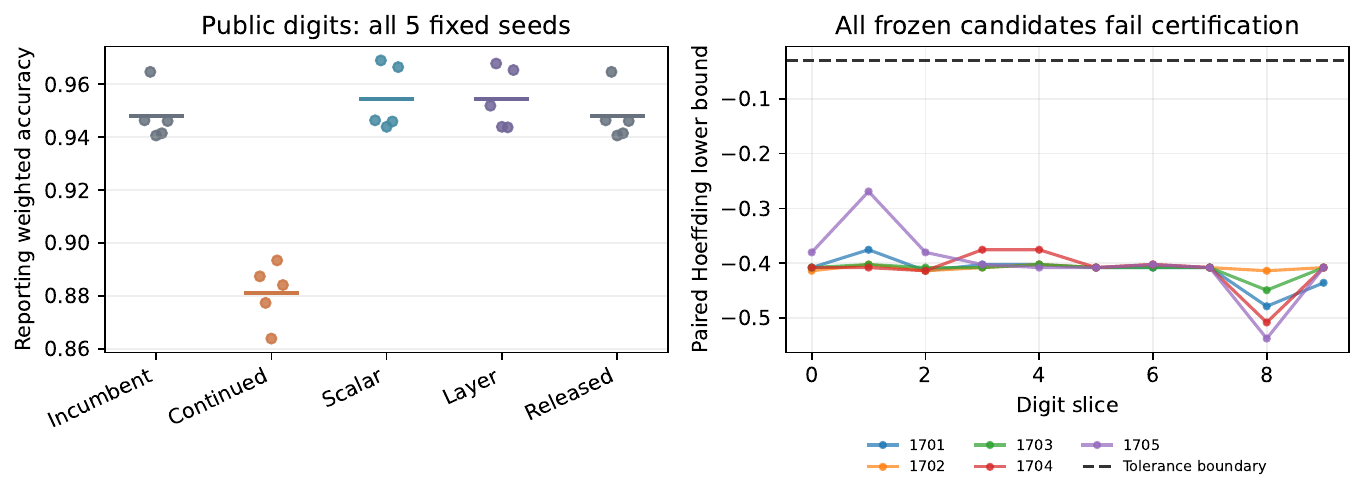}
\caption{Completed public-data integration study. The proposed merges improve substantially over the deliberately impaired continuation, but the release gate does not certify the frozen proposal in any run. Small per-class certification samples make fallback an expected and informative outcome.}
\end{figure}

\subsection{Why fallback is an informative result}
For a class with $n_s=36$, the generic paired Hoeffding radius is
$\sqrt{2\log(20)/36}\approx0.408$, far exceeding the 0.03 tolerance. The required empirical improvement is approximately $0.378$; for an incumbent class accuracy above 0.9, even a perfect candidate cannot achieve that improvement on the same examples. A candidate that agrees almost perfectly with a strong incumbent may have a near-zero observed difference yet fail this gate. For a hypothetical true difference of zero, merely making the radius smaller than 0.03 would require approximately $2\log(20)/0.03^2\approx6658$ iid examples per slice. This is a diagnostic calculation, not a sufficient condition for high joint power or a recommended fixed sample size.

These bounds can be conservative, and specialized paired classification tests or valid empirical-variance bounds may help. We did not revise the gate after observing these failures. This test primarily checks the data flow, common checkpoint ancestry, matched search counts, and exact fallback; the synthetic study is where population release probabilities are known and repeatedly assessed. The public digits sample is also not a translation model, a modern base-model upgrade, or evidence of layer-specific capability localization.

\subsection{Subsequent balanced continuation: a descriptive extension}
The selective continuation above reduces aggregate reporting accuracy, which limits its resemblance to a useful base-model upgrade. After reviewing that result, we specified a second protocol before executing it: keep the same seeds, splits, incumbent training, optimizer, objective weights, tolerance, search counts, tie rules, and gate, but continue on all 718 training examples. The implementation retrains each incumbent deterministically and checks its tensors and split indices against the original archive; all five reproduce exactly. The continuation optimizer starts with zero momentum buffers in both paths. Balanced continuation takes 144 mini-batch updates over 24 epochs, versus 72 in the selective path. The two paths are therefore not a compute-matched intervention on class coverage.

Table~\ref{tab:digits_paths} reports every policy from both paths. Balanced continuation improves mean overall and weighted reporting accuracy, yet both merge proposals have lower mean accuracy than the fully continued checkpoint. The layer proposal also has lower mean accuracy than the scalar proposal. These outcomes do not support layer superiority in this example. Mean accuracy improvement does not imply every seed or every class improves: Table~\ref{tab:digits_path_seeds} retains each seed and the chosen proposal's worst observed class difference. These are finite-reporting-set differences, not population certificates.

\begin{table}[ht]
\centering\footnotesize
\setlength{\tabcolsep}{3pt}
\begin{tabular}{llrrr}
\toprule
Path & Model & Accuracy & Weighted accuracy & Worst slice $\Delta$ \\
\midrule
Selective & Incumbent & $0.9478\pm0.0077$ & $0.9479\pm0.0098$ & $0.0000\pm0.0000$ \\
 & Continued checkpoint & $0.7300\pm0.0202$ & $0.8813\pm0.0113$ & $-0.9271\pm0.0464$ \\
 & Scalar proposal & $0.9494\pm0.0087$ & $0.9544\pm0.0123$ & $-0.0286\pm0.0495$ \\
 & Layer proposal & $0.9483\pm0.0089$ & $0.9546\pm0.0115$ & $-0.0456\pm0.0432$ \\
 & Released & $0.9478\pm0.0077$ & $0.9479\pm0.0098$ & $0.0000\pm0.0000$ \\
\midrule
Balanced & Incumbent & $0.9478\pm0.0077$ & $0.9479\pm0.0098$ & $0.0000\pm0.0000$ \\
 & Continued checkpoint & $0.9594\pm0.0107$ & $0.9637\pm0.0105$ & $-0.0114\pm0.0156$ \\
 & Scalar proposal & $0.9567\pm0.0109$ & $0.9587\pm0.0118$ & $-0.0113\pm0.0154$ \\
 & Layer proposal & $0.9544\pm0.0116$ & $0.9558\pm0.0122$ & $-0.0113\pm0.0154$ \\
 & Released & $0.9478\pm0.0077$ & $0.9479\pm0.0098$ & $0.0000\pm0.0000$ \\
\bottomrule
\end{tabular}

\caption{Both completed digits paths, mean $\pm$ descriptive seed SD. ``Worst slice $\Delta$'' is the per-seed minimum reporting class-accuracy difference from its incumbent, summarized over seeds. All frozen proposals fail; every released model is the incumbent. Balanced continuation was specified after the selective study and reuses its evaluation partitions.}
\label{tab:digits_paths}
\end{table}

\begin{table}[ht]
\centering\footnotesize
\setlength{\tabcolsep}{3pt}
\begin{tabular}{lrrrrlrr}
\toprule
Path/seed & Old & Updated & Scalar & Layer & Chosen $(\lambda_1,\lambda_2)$ & Min.\ $\Delta$ & Gate \\
\midrule
S/1701 & 0.9500 & 0.7361 & 0.9500 & 0.9556 & L $(0.100,0.000)$ & 0.0000 & Fail \\
S/1702 & 0.9528 & 0.7139 & 0.9556 & 0.9528 & S $(0.025,0.025)$ & 0.0000 & Fail \\
S/1703 & 0.9361 & 0.7111 & 0.9361 & 0.9361 & S $(0.067,0.067)$ & -0.0286 & Fail \\
S/1704 & 0.9444 & 0.7611 & 0.9472 & 0.9417 & L $(0.100,0.000)$ & -0.0556 & Fail \\
S/1705 & 0.9556 & 0.7278 & 0.9583 & 0.9556 & S $(0.117,0.117)$ & -0.1143 & Fail \\
\midrule
B/1701 & 0.9500 & 0.9611 & 0.9583 & 0.9472 & L $(0.000,1.000)$ & -0.0286 & Fail \\
B/1702 & 0.9528 & 0.9667 & 0.9667 & 0.9667 & S $(0.917,0.917)$ & 0.0000 & Fail \\
B/1703 & 0.9361 & 0.9472 & 0.9417 & 0.9417 & S $(0.592,0.592)$ & 0.0000 & Fail \\
B/1704 & 0.9444 & 0.9500 & 0.9500 & 0.9500 & L $(0.400,0.800)$ & -0.0278 & Fail \\
B/1705 & 0.9556 & 0.9722 & 0.9667 & 0.9667 & L $(0.900,0.400)$ & 0.0000 & Fail \\
\bottomrule
\end{tabular}

\caption{All ten seed/path outcomes. S and B denote selective and balanced continuation; the four accuracy columns use the reporting split. The chosen proposal is fixed using development observations; S/L in its column denotes scalar/layer interpolation. Coefficients are displayed to three decimals and saved at full precision. ``Min.\ $\Delta$'' is the chosen proposal's worst reporting class difference. Gate outcomes use the separate within-run certification split; no failed result was omitted.}
\label{tab:digits_path_seeds}
\end{table}

The balanced protocol and source/helper hashes were saved before its training began; original split indices, checkpoints, all development candidates, predictions, losses, and gate values are archived in \texttt{results/digits\_balanced/}. A separate verification script reconstructs reporting metrics and certification lower bounds from saved predictions for all ten seed/path runs. It also regenerates these tables. The original selective results are retained without modification.

This extension is subsequent and descriptive. Its certification and reporting partitions overlap the already examined selective study exactly within each seed, and the five seeds reuse one dataset. We do not treat the new outcomes as fresh confirmatory release certification, combine the two studies to claim a global error probability, or interpret seed variability as independent population uncertainty. The gate itself was executed unchanged, and no outcome triggered tuning. A future confirmatory deployment would require independent new data or a valid sequential design. Every balanced proposal also fails the conservative gate, so the additional path establishes neither a successfully certified public-model improvement nor a useful empirical merging advantage.

\section{A concrete next-stage foundation-model protocol}
This appendix specifies prospective work and contains no completed foundation-model result. First, choose at least two public, same-ancestry checkpoint transitions and verify their tensor mappings, token indices, licenses, and permitted redistribution. A version name or shared architecture is not sufficient. Include every prespecified pair in reporting; any path-quality exclusion must use development data and be disclosed.

Second, choose fixed language/domain slices and a bounded, interpretable per-example primary score. If corpus BLEU or a learned metric is also reported, use the appropriate paired document/source-cluster resampling and separate approximate inference from the finite-sample primary certificate. Multilingual translations of the same source should remain in the same split and resampling cluster. Set tolerances using application meaning, not the observed uncertainty of the candidate one wishes to release.

Third, compare incumbent, full update, scalar merge, block merge, a Pareto-front generator, and a repair/replay candidate under declared evaluation and compute budgets. Preserve the same incumbent and certification interface for every method. Development may tune methods within the budget; the certification split cannot serve as an iterative optimization oracle. If several methods can produce the finally released model, either select the method on development data or apply candidate-level multiplicity control.

Finally, record generation tokens, wall-clock evaluation cost, memory, training cost, split identifiers, seeds, candidate counts, gate outcomes, and independent reporting scores. Publish failed gates, negative results, and incompatible checkpoint pairs. A useful empirical claim would require accepted non-incumbent utility gains over matched baselines across these transitions, not only an attractive development frontier or an inability to detect harm.
\end{document}